\documentclass[final]{l4dc2026}

\title[Hierarchical End-to-End Taylor Bounds for Complete Neural Network Verification]{Hierarchical End-to-End Taylor Bounds for Complete Neural Network Verification}
\usepackage{times}

\author{%
 \Name{Taha Entesari} \Email{tentesa1@jhu.edu}\\
 \Name{Mahyar Fazlyab} \Email{mahyarfazlyab@jhu.edu}\\
 \addr Johns Hopkins University
}

\usepackage{xcolor}         
\usepackage{amsmath}
\usepackage{mathtools}

\usepackage{algorithm}
\usepackage{algorithmic}

\usepackage{booktabs}       
\usepackage{multirow}
\usepackage{caption}
\usepackage{wrapfig}

\newcommand{\norm}[1]{\left\| #1 \right\|}
\newcommand{\paran}[1]{\left( #1 \right)}
\newcommand{\relu}[1]{\left( #1 \right)_+}
\newcommand{\inverse}[1]{ #1^{-1}}

\newcommand{\RNum}[1]{\uppercase\expandafter{\romannumeral #1\relax}}

\begin{document}

\maketitle

\begin{abstract}%
 Reachability analysis of neural networks, which seeks to compute or bound the set of outputs attainable over a given input domain, is central to certifying safety and robustness in learning-enabled physical systems. Since exact reachable set computation is generally intractable, existing methods typically rely on tractable overapproximations. Examining the state of the art for smooth, twice-differentiable networks, we observe that existing approaches exploit at most second-order information and do not systematically leverage higher-order information. In this work, we introduce \textsc{HiTaB}, a novel verification framework that exploits second-order smoothness through both the Hessian, $\nabla^2 f$, and its Lipschitz constant, $L_{\nabla^2 f}$. We further develop a unified hierarchy of zeroth-, first-, and second-order bounds, together with precise conditions under which higher-order approximations yield provable improvements. Our main technical contribution is a compositional procedure for efficiently bounding $L_{\nabla^2 f}$ in deep neural networks via layerwise propagation of curvature bounds. We extend the framework to both $\ell_2$- and $\ell_\infty$-constrained input sets and show how it can be integrated into branch-and-bound verification pipelines. To our knowledge, this is the first practical reachability analysis framework for smooth neural networks that systematically exploits Lipschitz continuity of curvature, leading to tighter and more informative safety certificates.
\end{abstract}


\section{Introduction}
Ensuring the reliability of neural network-based systems is essential for their deployment in safety-critical applications, including autonomous driving \cite{ibrahum2024deep}, medical diagnosis \cite{javed2024robustness}, and control \cite{tambon2022certify}. A central verification question in this setting is the following: given a neural network and a bounded set of admissible inputs, can we efficiently compute or tightly bound the set of possible outputs? This problem arises in many settings, ranging from certifying robustness against adversarial perturbations \cite{wang2021beta} to verifying the safety of neural network controllers in closed-loop dynamical systems \cite{entesari2023reachlipbnb}. 

In this work, we study this question through the lens of reachability analysis. Given a neural network-represented function $f:\mathbb{R}^n \to \mathbb{R}$ and an input set $\mathcal{X}$, we consider the problem of upper bounding
\[
\max_{x \in \mathcal{X}} f(x).
\]
Such bounds provide certificates on the reachable output range of the network over $\mathcal{X}$ and form a core primitive in safety verification and robustness analysis.

Although this problem is generally nonconvex and computationally intractable, substantial progress has been made in developing tractable overapproximation methods that provide provable upper bounds on the optimal value. Such bounds enable sound, though conservative, certification of neural network behavior and underpin many modern formal verification frameworks \cite{zhang2018efficient, xu2020fast, wang2021beta, entesari2023reachlipbnb}.

Despite the extensive literature on verification and reachability analysis for non-smooth networks, particularly ReLU networks, comparatively less attention has been given to smooth and differentiable architectures. Existing methods specialized to differentiable networks include the branch-and-bound framework GenBaB \cite{shi2025neural} and several approaches that exploit derivative information to obtain certified bounds \cite{zhang2019recurjac, singla2020second, entesari2024compositional, sharifi2024provable}. However, these approaches rely on at most local first- or second-derivative information and do not systematically exploit higher-order smoothness to control approximation error.




\paragraph{Contributions}
In this work, we develop a practical framework for reachability analysis of smooth neural networks that systematically exploits higher-order smoothness. Our main contributions are:

\begin{itemize}
    \item \textbf{A compositional method for bounding the Hessian Lipschitz constant.}
    We develop a novel layerwise procedure for bounding the Lipschitz constant of the Hessian, $L_{\nabla^2 f}$, in scalar-valued feedforward neural networks. This provides a practical way to quantify higher-order smoothness at the network level.

    \item \textbf{An end-to-end second-order Taylor model with certified cubic remainder bounds.}
    We derive an end-to-end second-order Taylor approximation of the network that preserves the local gradient and Hessian of the network, and we bound the approximation error using the Lipschitz continuity of the Hessian. This yields a certified local upper bound whose remainder scales cubically with the input perturbation size.

    \item \textbf{A hierarchy of reachability bounds.}
    We develop a unified hierarchy of zeroth-, first-, and second-order reachability bounds, clarifying when higher-order information leads to tighter certificates.
\end{itemize}

Together, these components define \textsc{HiTaB} (\textbf{Hi}erarchical End-to-End \textbf{Ta}ylor \textbf{B}ounds), a framework that yields tighter reachability bounds for smooth neural networks and can be embedded within complete input space branch-and-bound verification pipelines.

\subsection{Related Work}

\paragraph{Reachability Analysis of Differentiable Networks}

Unlike piecewise linear neural networks, whose complete verification can be formulated using mixed-integer linear programs, smooth neural networks lack such general complete formulations. Consequently, much research has focused on incomplete reachability analysis.

These methods often rely on mathematical abstractions or optimization-based relaxations.
For instance, \cite{hu2020reach} presents a semidefinite program (SDP) approach for networks with slope-bounded activations, covering both piecewise linear and differentiable functions. 
ReachNN \cite{huang2019reachnn} utilizes Bernstein polynomials for Lipschitz networks, while \cite{kochdumper2023open} employs polynomial zonotopes to capture non-convex reachable sets for ReLU, sigmoid, and $\tanh$ functions. 
Other methods are highly specialized, such as \cite{ivanov2019verisig, ivanov2021verisig}, which target sigmoid and $\tanh$ activations specifically.
These methods generally follow a common pattern: abstract individual neurons and layers first, then compose them to obtain the end-to-end map. An alternative is to abstract the end-to-end map directly and then relax it. Along these lines, \cite{sharifi2024provable} recently proposed a framework that leverages the network's gradient information to derive an end-to-end abstraction via a first-order Taylor expansion. Our approach builds directly on this perspective, constructing an end-to-end \emph{third}-order Taylor abstraction that captures both gradient and curvature information.

Efforts to achieve complete verification for smooth networks often build upon these analyses. 
ReachLipBnB \cite{entesari2023reachlipbnb}, for example, wraps a Lipschitz-based analysis in a branching strategy to yield a complete verifier. Similarly, GenBaB \cite{shi2025neural} adapts the CROWN-like \cite{zhang2018efficient} Branch-and-Bound (BaB) framework, using element-wise bounding for smooth activations.


\paragraph{Lipschitz Estimation}
Lipschitz constant estimation for neural networks has been at the center of robustness analysis since the advent of adversarial attacks \cite{szegedy2013intriguing}. 
However, the spectral bound of \cite{szegedy2013intriguing} proved to be highly conservative with limited practical utility. 
Subsequent works focused on obtaining local Lipschitz constants, which can significantly reduce conservatism at the expense of increased computation, either through bound propagation techniques \cite{huang2021training, shi2022efficiently} or more intensive mixed-integer linear programs \cite{jordan2020exactly}. 
Later, \cite{fazlyab2019efficient} introduced LipSDP, an SDP to calculate global and local Lipschitz constants for neural networks with \emph{slope-bounded} activations. 
This framework inspired a wide range of follow-up works, expanding and specializing the LipSDP algorithm to specific architectures \cite{pauli2023lipschitz}, exploiting the structure of layers \cite{araujo2023unified, wang2023direct}, extending the methodology to certain non-slope-restricted activation functions \cite{pauli2024novel}, and providing improved spectral-like bounds \cite{fazlyab2023certified}.\\
Beyond the Lipschitz constant of a neural network, estimating the Lipschitz constant of the gradient (or equivalently, the norm of the Hessian) has also received attention. 
\cite{singla2020second} first addressed this problem and provided an algorithm for computing bounds on the Hessian of scalar logits of a feedforward network. 
\cite{sharifi2024provable} studied local versions of this bound.
\cite{entesari2024compositional} presented a different approach, providing a formulation for calculating the Lipschitz constant of the gradient of general function compositions, with specializations to neural networks. 
Here, we go beyond these and bound the Lipschitz constant of the \textit{second} derivative. 

\section{Background and Problem Statement} \label{sec:backgroundUpperBounding}

We consider a twice-differentiable function $f:\mathbb{R}^n \to \mathbb{R}$ indexed by a neural network and the problem of upper bounding its maximum over a Euclidean ball centered at a nominal point $x_c \in \mathbb{R}^n$:
\begin{align}
    \label{eq:originalProblemFormulation}
    \max_{x \in \mathcal{B}(x_c,\varepsilon)} \; f(x), \quad \text{where} \quad \mathcal{B}(x_c,\varepsilon) \coloneq \{x \in \mathbb{R}^n \mid \|x-x_c\|_2 \le \varepsilon\}.
\end{align}
This problem arises in a variety of applications, including certification against adversarial perturbations in neural network classifiers and reachability analysis of learning-enabled dynamical systems. In general, \eqref{eq:originalProblemFormulation} is nonconvex and computationally intractable to solve exactly. Accordingly, our goal is to derive tractable upper bounds on its optimal value. To this end, we seek a provable pointwise majorizer $\bar f$ such that
\begin{equation}
    \label{eq:generalPointWiseUpperBound}
    f(x_c+\delta) \le \bar f(x_c,\delta),
    \qquad \forall \delta \text{ with } \|\delta\|_2 \le \varepsilon.
\end{equation}
The majorizer $\bar f$ is constructed using local smoothness information at $x_c$, namely the function value $f(x_c)$, gradient $\nabla f(x_c)$, Hessian $\nabla^2 f(x_c)$, and their Lipschitz bounds. Given such a pointwise upper bound, we obtain the induced global certificate
\begin{equation}
    \label{eq:generalPointWiseUpperBoundWithSup}
    f^*(x_c,\varepsilon)
    \coloneq
    \sup_{\|\delta\|_2 \le \varepsilon} f(x_c+\delta)
    \le
    \sup_{\|\delta\|_2 \le \varepsilon} \bar f(x_c,\delta)
    \eqqcolon
    \bar f^*(x_c,\varepsilon).
\end{equation}
For practical verification, we require $\bar f$ to be amenable to optimization over $\mathcal{B}(x_c,\varepsilon)$, ideally yielding a closed-form or efficiently computable upper bound on $\bar f^*(x_c,\varepsilon)$.

Such upper bounds are useful in applications including neural network verification \cite{singla2020second, fazlyab2023certified} and reachable-set overapproximation \cite{entesari2023reachlipbnb, sharifi2024provable}. In this work, we unify prior constructions of $\bar f$, which are typically based on zeroth-, first-, or second-order local models with bounded remainder terms, and strengthen them by incorporating Hessian Lipschitz continuity. This yields tighter certificates for smooth neural networks by controlling the cubic remainder of an end-to-end second-order Taylor model.

\section{Hierarchy of End-to-End Taylor Models}
In this section, we develop a hierarchy of upper bounds for \eqref{eq:originalProblemFormulation} using local information of increasing order. We index the hierarchy by the highest-order derivative information used in the construction, rather than by the degree of the resulting majorizer.

\subsection{Zeroth-Order Information}
The simplest majorizer uses only the local function value and the Lipschitz constant $L_f$. In this case, Lipschitz continuity directly gives
\[
f(x_c+\delta) \le \bar f_0(x_c,\delta) \coloneq f(x_c) + L_f \|\delta\|_2.
\]
Maximizing over the perturbation set $\{\delta : \|\delta\|_2 \le \varepsilon\}$ yields the corresponding worst-case certificate
\[
\sup_{\|\delta\|_2 \le \varepsilon} f(x_c+\delta)
\le
\bar f_0^*(x_c,\varepsilon)
\coloneq
f(x_c) + L_f \varepsilon.
\]
This majorizer only preserves the function value at $x_c$.

\subsection{First-Order Information}

The zeroth-order bound can be improved by incorporating first-order local information, namely the gradient $\nabla f(x_c)$ and the gradient Lipschitz constant $L_{\nabla f}$. Using a first-order Taylor expansion with a quadratic remainder bound gives the pointwise estimate
\[
f(x_c+\delta)
\le
\bar f_1(x_c,\delta)
\coloneq
f(x_c)+\nabla f(x_c)^\top \delta + \frac{1}{2}L_{\nabla f}\|\delta\|_2^2.
\]
This majorizer preserves the function value and gradient at $x_c$. Maximizing the upper bound over $\|\delta\|_2 \le \varepsilon$ yields the closed-form certificate
\[
\sup_{\|\delta\|_2 \le \varepsilon} f(x_c+\delta)
\le
\bar f_1^*(x_c,\varepsilon)
\coloneq
f(x_c)+\|\nabla f(x_c)\|_2\,\varepsilon+\frac{1}{2}L_{\nabla f}\varepsilon^2.
\]
See \cite{entesari2024compositional} for derivations and proofs. Compared to the zeroth-order bound, the first-order bound is tighter whenever
\begin{equation}
    \label{eq:whenFirstIsBetter}
    \varepsilon \le \frac{2(L_f-\|\nabla f(x_c)\|_2)}{L_{\nabla f}}.
\end{equation}
We note that the threshold is always nonnegative since $L_f \geq \sup_{x} \|\nabla f(x)\|_2$. This condition makes explicit when incorporating gradient information is beneficial: the improvement is most pronounced when the local slope, captured by $\|\nabla f(x_c)\|_2$, is significantly smaller than the global Lipschitz constant $L_f$ of $f$, while the local curvature, as quantified by $L_{\nabla f}$, remains moderate.

\subsection{Second-Order Information}
To further tighten the upper bound on $f(x_c+\delta)$, we incorporate second-order local information, namely the Hessian $\nabla^2 f(x_c)$ together with the Hessian Lipschitz constant $L_{\nabla^2 f}$. This yields the pointwise upper bound
\begin{equation}
\label{eq:pointwise2ndOrderBound}
f(x_c+\delta)
\le
\bar f_2(x_c,\delta)
\coloneq
f(x_c)
+ \nabla f(x_c)^\top \delta
+ \frac{1}{2}\delta^\top \nabla^2 f(x_c)\delta
+ \frac{1}{6}L_{\nabla^2 f}\|\delta\|_2^3.
\end{equation}
That is, $\bar f_2$ is an end-to-end second-order Taylor model augmented with a certified cubic remainder term. We provide a formal derivation of \eqref{eq:pointwise2ndOrderBound} in Appendix~\ref{proofOfeq:pointwise2ndOrderBound}. By capturing both local curvature and higher-order smoothness, this bound can yield substantially tighter certificates, particularly in regions where $f$ is well approximated by a quadratic model.\\
In contrast to $\bar f_0$ and $\bar f_1$, the second-order majorizer generally does not admit a closed-form worst-case bound:
$
\bar f_2^*(x_c,\varepsilon)
\coloneq
\sup_{\|\delta\|_2 \le \varepsilon} \bar f_2(x_c,\delta).
$
This optimization is generally intractable because it involves maximizing a nonconvex quadratic term together with a cubic remainder over a Euclidean ball. Nonetheless, as we show next, it is possible to derive informative and computationally tractable upper bounds on $\bar f_2^*$ by exploiting structure in the quadratic form or spectral properties of $\nabla^2 f(x_c)$.

\begin{proposition}[Split and Bound]\label{prop:splitAndBound}
    Let $\bar{f}_2^\text{sb}(x_c, \varepsilon)$ be defined as 
    \[
    \bar{f}_2^\text{sb}(x_c, \varepsilon) = f(x_c) + \norm{\nabla f(x_c)}_2 \varepsilon + \frac{1}{2} \relu{\lambda_{\max} (\nabla^2 f(x_c))} \varepsilon^2 + \frac{1}{6}L_{\nabla^2 f} \varepsilon^3,
    \]
    where $\relu{a}=\max(a,0)$. Then $\bar{f}_2^*(x_c, \varepsilon) \leq \bar{f}_2^\text{sb}(x_c, \varepsilon)$.
\end{proposition}
We defer the proof to Appendix \ref{proofOfProp:splitAndBound}.
The closed-form certificate $\bar f_2^{\mathrm{sb}}(x_c,\varepsilon)$ provides a practical upper bound on the second-order worst-case value $\bar f_2^*(x_c,\varepsilon)$. Moreover, it allows us to characterize regimes, analogous to \eqref{eq:whenFirstIsBetter}, in which the bound using second-order information is provably tighter than the corresponding first- and zeroth-order bounds. As expected, this improvement is most pronounced for small perturbation budgets, where the cubic remainder is dominated by lower-order terms. We make this statement precise in the next result and defer the proof to Appendix~\ref{proofOfProp:thirdIsBetter}.

\begin{proposition}
\label{prop:thirdIsBetter}
The split-and-bound certificate $\bar f_2^{\mathrm{sb}}(x_c,\varepsilon)$ satisfies $\bar f_2^{\mathrm{sb}}(x_c,\varepsilon)\le \bar f_1^*(x_c,\varepsilon)$
whenever
\[
\varepsilon \leq \frac{3}{L_{\nabla^2 f}}\left( L_{\nabla f} - \relu{\lambda_{\max} \nabla^2 f(x_c)} \right)
\]
\end{proposition}
Note that the threshold in Proposition~\ref{prop:thirdIsBetter} is always nonnegative, since $L_{\nabla f} \ge \sup_x \lambda_{\max}(\nabla^2 f(x))$. 

Moreover, once $\bar f_2^{\mathrm{sb}}(x_c,\varepsilon)$ has been computed, the corresponding first- and zeroth-order certificates, $\bar f_1^*(x_c,\varepsilon)$ and $\bar f_0^*(x_c,\varepsilon)$, can be obtained with essentially no additional overhead, since all required quantities are already available. This observation allows us to combine the hierarchy into a single certificate that is guaranteed to be at least as tight as each individual bound:
\[
\bar f_m(x_c,\varepsilon)
\coloneq
\min\bigl\{
\bar f_0^*(x_c,\varepsilon),\,
\bar f_1^*(x_c,\varepsilon),\,
\bar f_2^{\mathrm{sb}}(x_c,\varepsilon)
\bigr\}.
\]
By construction, $\bar f_m(x_c,\varepsilon)$ dominates all three certificates and therefore never performs worse than any of the individual methods.

\begin{remark} \normalfont
Proposition~\ref{prop:splitAndBound} yields a practical closed-form upper bound on $\bar{f}_2^*(x_c,\varepsilon)$. In principle, one can obtain less conservative bounds by replacing the split-and-bound relaxation with tighter numerical procedures, which typically require solving a nonconvex optimization problem or a system of nonlinear equations. In our experiments, however, the resulting gains were modest relative to the additional computational burden. For this reason, we adopt the closed-form certificate here and leave the design of tighter numerical relaxations to future work.
\end{remark}
To make the proposed bound operational, it remains to estimate the Hessian Lipschitz constant $L_{\nabla^2 f}$. We develop this construction in Section~\ref{sec:lipschitzOfHessianEstimation}. Before turning to that problem, however, we derive the analogous upper bounds for the case in which the input region is an $\ell_\infty$ ball.

\subsection{$\ell_\infty$-Bounded Inputs}

As established in \cite{entesari2024compositional}, one can derive $\ell_\infty$ counterparts of the upper bounds $\bar f_i$. Doing so directly, however, requires Lipschitz constants with respect to the $\ell_\infty$ norm. Since the methods developed later in Section~\ref{sec:lipschitzOfHessianEstimation} are primarily geared toward estimating $\ell_2$ Lipschitz constants, we reuse the previously established pointwise upper bounds $\bar f_i$ and convert them into valid certificates over $\ell_\infty$ input balls via norm relations.


We consider generalized $\ell_\infty$-bounded perturbations and study the problem
\begin{equation}
    \label{eq:inftyBoundedProblem}
    \max_{\delta} \; f(x_c+\delta)
    \qquad
    \text{subject to}
    \qquad
    \|D^{-1}\delta\|_\infty \le 1,
\end{equation}
where $D$ is a diagonal positive definite matrix. The feasible set is an axis-aligned box, and the standard $\ell_\infty$ perturbation model is recovered by taking $D=\varepsilon I$. Analogously to the $\ell_2$-bounded case in \eqref{eq:generalPointWiseUpperBoundWithSup}, we define
\begin{equation}
    f^*(x_c,D)
    \coloneq
    \sup_{\|D^{-1}\delta\|_\infty \le 1} f(x_c+\delta)
    \le
    \sup_{\|D^{-1}\delta\|_\infty \le 1} \bar f(x_c,\delta)
    \eqqcolon
    \bar f^*(x_c,D).
\end{equation}
As established in \cite{entesari2023reachlipbnb}, the zeroth-order certificate takes the form
$
\bar f_0^*(x_c,D)=f(x_c)+L_f\|D\|_F,
$
where $\|\cdot\|_F$ denotes the Frobenius norm. Moreover, \cite{sharifi2024provable} shows that the first-order certificate is given by
$
\bar f_1^*(x_c,D)
=
f(x_c)+\|D\nabla f(x_c)\|_1+\frac{L_{\nabla f}}{2}\|D\|_F^2.
$
We next derive a counterpart of the second-order certificate for generalized $\ell_\infty$ perturbations using the same split-and-bound idea. The proof is deferred to Appendix~\ref{sec:proofOfSecondOrderUpperBoundInfinity}.
\begin{proposition}\label{prop:secondOrderUpperBoundInfinity}
    Let $\bar{f}_2^\text{sb}(x_c, D)$ be defined as 
    \[
    \bar{f}_2^\text{sb}(x_c, D) =  f(x_c) 
             +
              \|D \nabla f(x_c)\|_1
             +
             \frac{1}{2}n \paran{\lambda_{\max} \paran{D\nabla^2 f(x_c) D}}_+
             +
             \frac{L_{\nabla^2 f}}{6} \norm{D}_F^3,
    \]
    where $\relu{a}=\max(a,0)$. Then $\bar{f}_2^*(x_c, D) \leq \bar{f}_2^\text{sb}(x_c, D)$.
\end{proposition}



\section{Efficient Estimation of the Lipschitz Constant of the Hessian}
\label{sec:lipschitzOfHessianEstimation}

To make the second-order certificate practical, it remains to compute a tractable upper bound on the Hessian Lipschitz constant $L_{\nabla^2 f}$. In this section, we develop such a bound for smooth fully connected neural networks. Specifically, we consider an $L$-layer twice-differentiable network $x \mapsto z^L(x)$ with input dimension $N_0$ and output dimension $N_L$, and consider scalar outputs of the form 
$f(x)=c^\top z^L(x).$
For an input $x \in \mathbb{R}^{N_0}$, let $z^I(x) \in \mathbb{R}^{N_I}$ and $a^I(x) \in \mathbb{R}^{N_I}$ denote the pre-activation and post-activation vectors at layer $I$, respectively, with $a^0(x)=x$. When no confusion arises, we omit the explicit dependence on $x$. The network evolves according to 
\[
z^I = W^I a^{I-1} + b^I,
\qquad
a^I = \sigma(z^I),
\]
where $W^I \in \mathbb{R}^{N_I \times N_{I-1}}$ and $b^I \in \mathbb{R}^{N_I}$ are the weight matrix and bias vector of layer $I$, and $\sigma$ is an elementwise twice-differentiable activation function. We assume that $\sigma$, $\sigma'$, and $\sigma''$ are Lipschitz continuous, with Lipschitz constants $L_{\sigma}$, $L_{\sigma'}$, and $L_{\sigma''}$.

It has been established that for this architecture, $z^L$ and $\nabla z^L$ are Lipschitz continuous and one can find certified Lipschitz constants $L_{z^L}$ and $L_{\nabla z^L}$ for them, respectively, \cite{entesari2024compositional, singla2020second}. 
Here, we establish an algorithm to acquire a bound on $L_{\nabla^2 z^L}$. 
To the best of our knowledge, this is the first work to establish a computationally tractable algorithm to provide upper bounds on the Lipschitz constant of the Hessian of a neural network.

To derive a tractable bound on the Hessian Lipschitz constant, we exploit the compositional structure of the network. In particular, we first establish a general composition rule for the Lipschitz constant of the Hessian, and then specialize it to feedforward neural networks. For each layer $I$, define $F^I:\mathbb{R}^{N_{I-1}}\to\mathbb{R}^{N_I}$ by $F^I(x)=\sigma(W^I x+b^I)$, and for the final layer let $F^L(x)=W^Lx+b^L$. Then $a^I(x)=F^I\circ a^{I-1}(x)$. Let $F_j^I$ denote the $j$th coordinate of $F^I$. By the chain rule, $D(F_j^I\circ a^{I-1})(x)=\nabla F_j^I(a^{I-1}(x))^\top D a^{I-1}(x)$, and
\[
D^2(F_j^I\circ a^{I-1})(x)
=
D a^{I-1}(x)^\top \nabla^2 F_j^I(a^{I-1}(x)) D a^{I-1}(x)
+
\sum_{i=1}^{N_{I-1}}
\frac{\partial F_j^I}{\partial x_i}\bigl(a^{I-1}(x)\bigr)\,\nabla^2 a_i^{I-1}(x),
\]
where $D$ denotes the differential operator. 
The following theorem establishes our main compositional bound; we defer the proof to Appendix \ref{sec:proofOfLipschitzOfCurvatureCompositional}.
\begin{theorem} \label{thm:lipschitzOfCurvatureCompositional}
    Define $L_{\nabla^2 a_j^I}$ as follows
    %
    \begin{align*}
L_{\nabla^2 a_j^I}
&\coloneq
L_{\nabla^2 F_j^I}\,L_{a^{I-1}}^3
+
2\,L_{D a^{I-1}}\,L_{a^{I-1}}\,L_{\nabla F_j^I} \\
&\quad
+
\sum_{i=1}^{N_{I-1}}
\left(
L_{\partial_i F_j^I}\,L_{a^{I-1}}\,L_{D a_i^{I-1}}
+
L_{\nabla^2 a_i^{I-1}}
\sup_x \bigl|\partial_i F_j^I(x)\bigr|
\right),
\end{align*}
where $\partial_i F_j^I(x) \coloneq \frac{\partial F_j^I(x)}{\partial x_i}$. Then $L_{\nabla^2 a_j^I}$ is a Lipschitz constant of the Hessian of the map $x \mapsto a_j^I(x)$.
\end{theorem}
Theorem~\ref{thm:lipschitzOfCurvatureCompositional} yields a modular procedure for bounding the Hessian Lipschitz constant of each layer. Moreover, all quantities appearing in the theorem can be computed efficiently using existing tools. In particular, Lipschitz constants of the activations $a^I$ can be bounded using \cite{fazlyab2023certified, fazlyab2019efficient}, while Lipschitz constants of their gradients, namely $L_{Da^I}$ and $L_{Da_i^I}$, can be obtained using \cite{entesari2024compositional}. Moreoever, for the layer map $F^I(x)=\sigma(W^I x+b^I)$, we have $DF^I(x)=\mathrm{Diag}(\sigma'(W^I x+b^I))W^I$. Consequently,

\begin{minipage}{.48\textwidth}
\begin{equation}
\label{eq:supOfDerivative}
\sup_x \left|\frac{\partial F_j^I(x)}{\partial x_i}\right|
=
L_\sigma |W^I_{ji}|
\end{equation}
\end{minipage}
\hfill
\begin{minipage}{.48\textwidth}
\begin{equation}
\label{eq:lipschitzOfPartialDerivative}
L_{\partial_i F_j^I}
=
L_{\sigma'} \|W^I_{j,:}\|_2 |W^I_{ji}|
\end{equation}
\end{minipage}

\begin{minipage}{.48\textwidth}
\begin{equation}
\label{eq:lipschitzOfGradient}
L_{\nabla F_j^I}
=
L_{\sigma'} \|W^I_{j,:}\|_2^2
\end{equation}
\end{minipage}
\hfill
\begin{minipage}{.48\textwidth}
\begin{equation}
\label{eq:lipschitzOfCurvature}
L_{\nabla^2 F_j^I}
=
L_{\sigma''} \|W^I_{j,:}\|_2^3
\end{equation}
\end{minipage}
where $\partial_i F_j^I \coloneq \frac{\partial F_j^I}{\partial x_i}$. The derivation of \eqref{eq:supOfDerivative}--\eqref{eq:lipschitzOfCurvature} is deferred to Appendix~\ref{sec:appendixIndividualLipschitzTerms}. Together, these expressions allow the quantities $L_{\nabla^2 a_i^I}$ to be computed sequentially across layers, ultimately yielding bounds on $L_{\nabla^2 z_j^L}$.

To bound $L_{\nabla^2 f}$ for $f(x)=c^\top z^L(x)$, one could directly replace the final affine layer $F^L(x)=W^Lx+b^L$ by $\hat F^L(x)=c^\top W^Lx+c^\top b^L$ and apply the same recursion. However, a tighter bound can be obtained by combining the final two layers and instead considering
\[
\hat F(x)=c^\top W^L \sigma(W^{L-1}x+b^{L-1})+c^\top b^L.
\]
Given all quantities up to layer $L-2$, we then apply the same compositional bound to $\hat F$. Defining
$A \coloneq (W^{L-1})^\top \mathrm{Diag}(c^\top W^L)$, 
we obtain

\begin{minipage}{.48\textwidth}
\begin{equation*}
\sup_x \left|\frac{\partial \hat F(x)}{\partial x_i}\right|
=
L_\sigma \sum_j |A_{ij}|
\end{equation*}
\end{minipage}
\hfill
\begin{minipage}{.48\textwidth}
\begin{equation*}
L_{\partial_i \hat F}
=
L_{\sigma'} \sum_j |A_{ij}| \|W^{L-1}_{j,:}\|_2
\end{equation*}
\end{minipage}
and
$
L_{\nabla \hat F}
=
L_{\sigma'} \|A\|_2 \|W^{L-1}\|_2.
$
For $L_{\nabla^2 \hat F}$, we use Theorem 3.8 of \cite{entesari2024compositional}, which yields a tighter estimate than the naive bound
$
L_{\nabla^2 \hat F}
=
L_{\sigma''}\|A\|_2 \|W^{L-1}\|_2 \max_i \|W^{L-1}_{i,:}\|_2.
$
The derivation of these bounds is deferred to Appendix~\ref{sec:appendixIndividualLipschitzTermsForSecondFunction}.

%
\section{Integration of Components: The Final Algorithm}

We now combine the ingredients developed in the previous sections into a practical algorithm for the generalized $\ell_\infty$-bounded problem \eqref{eq:inftyBoundedProblem}. Given the Lipschitz constants computed as in Section~\ref{sec:lipschitzOfHessianEstimation}, we first apply the upper-bounding framework of Section~\ref{sec:backgroundUpperBounding} to construct the certificate $\bar f_m(x_c,D)$. 

When this bound is not sufficiently tight for the application at hand, we further embed the method within a branch-and-bound (BaB) framework. Specifically, we adopt the BaB procedure of \cite{entesari2023reachlipbnb}, which iteratively partitions the input domain to reduce the relaxation gap until a prescribed tolerance is met. Our top-level procedure is summarized in Algorithm~\ref{alg:branchAndBound}. 

\begin{algorithm}[h!]
\caption{Hierarchical End-to-End Taylor Bounds} 

\begin{algorithmic} 
\STATE \textbf{Input:} Desired point $x_c$, input box $D_0$, function $f(\cdot)$, gradient calculator $\nabla f(\cdot)$, vector-hessian calculator $h(x, v) \coloneq \nabla^2 f(x)v$, Stateful branching algorithm $\mathcal{B}$, tolerance $\varepsilon_\text{tol}$
\STATE \textbf{Output:} Upper Bound on $\max_{\norm{\inverse{D_0}\delta}_\infty \leq 1} f(x_c + \delta)$.
\STATE \textbf{Initialize:} $x \gets x_c, D \gets D_0$, Gap $\gets \infty$, Initialize $\mathcal{B}(f, x_c, D_0)$
\STATE $(L_f, L_{a^I}) \gets $  \cite{fazlyab2023certified} $(f)$
\STATE $(L_{\nabla f}, L_{Da^I}, L_{Da^I_i}) \gets $  \cite{entesari2024compositional} $(f, L_{a^I})$
\STATE $(L_{\nabla^2 f}, L_{\nabla^2 a^I_j}) \gets$ Theorem \ref{thm:lipschitzOfCurvatureCompositional} $(f, L_{a^I}, L_{Da^I}, L_{Da^I_i})$

\WHILE {Gap $> \varepsilon_\text{tol}$} 
\STATE $\bar{f}_0^* \gets f(x) + L_f \norm{D}_F \qquad$ and $\qquad \bar{f}_1^* \gets f(x) + \|D\nabla f(x)\|_1 + \frac{L_{\nabla f}}{2} \norm{D}_F^2$
\STATE $\lambda \gets \lambda_{\max}(D\nabla^2 f(x) D)$
\STATE $\bar{f}_2^\text{sb} \gets  f(x) 
             +
              \|D \nabla f(x)\|_1
             +
             \frac{1}{2}n \paran{\lambda}_+
             +
             \frac{L_{\nabla^2 f}}{6} \norm{D}_F^3$
\STATE $\bar{f}_m \gets \min\{\bar{f}_0^*, \bar{f}_1^*, \bar{f}_2^\text{sb}\}$
\STATE $(x, D, Gap) \gets \mathcal{B}(x, D, \bar{f}_m)$. \COMMENT{Given the current upper bound, the branching algorithm will provide the problem that needs to be solved at the next step.}
\ENDWHILE

\STATE \textbf{Return:} $\bar{f}_m$
\end{algorithmic}
\label{alg:branchAndBound}
\end{algorithm}

The remaining ingredient in the proposed algorithm is the computation of $\lambda_{\max}\!\bigl(D\nabla^2 f(x)D\bigr)$. Several approaches are available for this task. A straightforward option is to use modern automatic differentiation frameworks to form the Hessian $\nabla^2 f(x)$ explicitly, form the scaled matrix $D\nabla^2 f(x)D$, and then compute its largest eigenvalue using standard eigendecomposition routines. This approach is practical for smaller networks and moderate input dimensions.

When the Hessian is too expensive to compute explicitly, one may instead rely on Hessian--vector products, which are supported efficiently in modern automatic differentiation frameworks, together with iterative methods such as power iteration or Lanczos \cite{van1996matrix}, to estimate $\lambda_{\max}\!\bigl(D\nabla^2 f(x)D\bigr)$ directly. In our experiments, the first approach yielded better runtime and was therefore used throughout.

\section{Experiments}

In this section, we present numerical experiments on reachability analysis under $\ell_\infty$-bounded state perturbations for LTI systems controlled by neural networks. Specifically, we consider the optimization problem
\begin{align}
    \label{eq:reachabilityAnalsisProblem}
    \max_x \quad c^\top \bigl(Ax + B z^L(x) + d\bigr) \quad 
    \text{s.t.} \quad \|D(x-x_c)\|_\infty \le 1,
\end{align}
where $A \in \mathbb{R}^{n_0 \times n_0}$, $B \in \mathbb{R}^{n_0 \times n_u}$, and $d \in \mathbb{R}^{n_0}$ define the system dynamics, and $n_u$ denotes the number of control inputs. The neural network controller $z^L$ is trained to imitate a model predictive controller (MPC).

Our goal is to bound the set of reachable states of the closed-loop system over multiple time steps and verify that the system can both avoid unsafe regions and reach its target state. By solving \eqref{eq:reachabilityAnalsisProblem} over a suitable set of directions $c$, we obtain a polyhedral outer approximation of the reachable set at the next time step.  This approximation can then be propagated recursively for multi-step reachability analysis. Different choices of directions yield different geometric templates; in this work, we use rotated rectangles, following \cite{entesari2023reachlipbnb}.

We consider the standard 6D quadrotor benchmark. For a discretization step $\Delta t$, the system dynamics are given by
\[
A = I_{6\times6} + \!\Delta t \times \begin{bmatrix}
0_{3\times3} & I_{3\times3} \\
0_{3\times3} & 0_{3\times3}
\end{bmatrix}, 
\quad 
B = \!\Delta t \!\times  \!\begin{bmatrix}
& g & 0 & 0 \\
0_{3\times3} & 0 & -g & 0 \\
& 0 & 0 & 1
\end{bmatrix}^\top,
\quad
d = \Delta t \times \begin{bmatrix}
0_{5\times1} \\
-g
\end{bmatrix}
\]
\begin{wrapfigure}[17]{r}{0.5\textwidth}
\includegraphics[width=.45\textwidth]{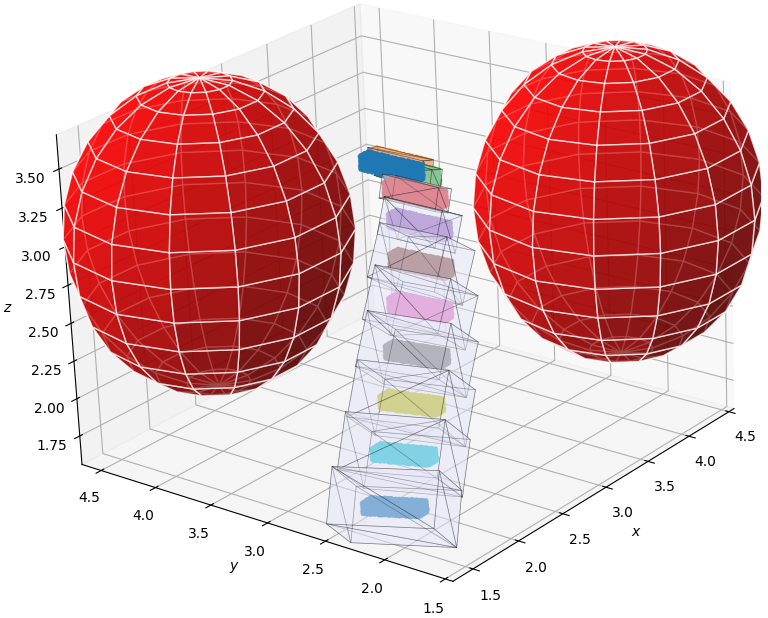}
        \caption{The quadrotor problem setup. The point clouds show trajectory samples from the system via numerical simulation. The obstacles are shown as two spheres. The reachable sets are calculated with a tolerance of $0.001$.}
        \label{fig:quadrotorReacability}
\end{wrapfigure}
The control input is $u =
\begin{bmatrix}
\tan(\theta) & \tan(\phi) & \tau
\end{bmatrix}^\top$, 
which in our setting is produced by the neural network controller $z^L$. Following \cite{sharifi2024provable}, training data is collected from an MPC controller that steers the system toward the origin while avoiding two spherical obstacles; see Figure~\ref{fig:quadrotorReacability}. A neural network with architecture $6 \times 32 \times 32 \times 3$ and $\tanh$ activations is then trained to imitate the MPC controller.\\
We consider a hyperrectangular initial set and compute its reachable sets over 10 time steps, corresponding to a horizon of 1.0 seconds.


\begin{wrapfigure}[10]{r}{0.5\textwidth}
  \centering
    \scalebox{.73}{
    \begin{tabular}{c c c c}
    \toprule
         BaB Tol & Bounding Method & Branches & Run Time (s) \\
         \midrule
          \multirow{2}{*}{$10^{-2}$}  & 1st order (local) & 908.0k $\pm$ 36.5k & 320.4 $\pm$ 13.4\\
          & \textsc{HiTaB} (Ours) & 119.7k $\pm$ 9.3k & 108.7 $\pm$ 11.5\\
          \midrule
         \multirow{2}{*}{$10^{-3}$} & 1st order (local) & 1613k $\pm$ 113.8k & 555.0 $\pm$ 44.6 \\
         & \textsc{HiTaB} (Ours) & 434.8k $\pm$ 26.4k & 531.4 $\pm$ 37.1 \\
         \bottomrule
    \end{tabular}
    }
    \captionof{table}{Run times and number of branches for the BaB algorithm using various bounding algorithms and tolerances. We report statistics of mean $\pm$ standard deviation over 10 runs.}
    \label{tab:comparison}
\end{wrapfigure}

We compare our method against the closest baseline, namely the first-order local bounding approach of \cite{sharifi2024provable}. For both methods, the Lipschitz constants $L_{a^I}$ are computed using LipLT \cite{fazlyab2023certified}. The results are reported in Figure~\ref{fig:quadrotorReacability} and Table~\ref{tab:comparison}.

As shown in Table~\ref{tab:comparison}, the bound using second-order information substantially reduces the number of branches generated by the BaB procedure at both tolerance levels. This directly reflects the tighter local approximation provided by the second-order bound, especially on smaller subdomains where higher-order information becomes most effective.
\vspace{-5mm}
\section{Conclusion}

In this work, we introduced \textsc{HiTaB}, a practical framework for verification of smooth neural networks that systematically exploits Hessian Lipschitz continuity. Our approach addresses a key limitation of existing verification methods by incorporating second-order smoothness information, leading to tighter certificates for smooth architectures. We developed a hierarchy of bounds based on zeroth-, first-, and second-order information, together with explicit conditions under which higher-order constructions provably improve upon lower-order ones. Our main technical contribution is a compositional method (Theorem~\ref{thm:lipschitzOfCurvatureCompositional}) for efficiently upper-bounding $L_{\nabla^2 f}$ in deep neural networks through layerwise propagation of curvature information. We further extended the second-order framework to $\ell_\infty$-bounded perturbations (Proposition~\ref{prop:secondOrderUpperBoundInfinity}) and integrated it into a branch-and-bound pipeline through Algorithm~1.\\
\textbf{Limitations and Future Work.}
Our experiments focused on $\ell_\infty$-bounded problems.
In the $\ell_2$ setting, tighter estimates of $\bar f_2^*(x_c,\varepsilon)$ may be achievable, albeit at greater computational cost, for example by solving the associated optimality conditions numerically. Exploring such refinements, together with a broader empirical study of second-order bounds under different input geometries, is an important direction for future work. In addition, our current implementation is CPU-based and does not exploit GPU parallelism. Adapting the major components of the algorithm to GPU architectures could substantially improve runtime and scalability.

\clearpage

\bibliography{bib}

@inproceedings{singla2020second,
  title={Second-order provable defenses against adversarial attacks},
  author={Singla, Sahil and Feizi, Soheil},
  booktitle={International conference on machine learning},
  pages={8981--8991},
  year={2020},
  organization={PMLR}
}

@article{fazlyab2023certified,
  title={Certified robustness via dynamic margin maximization and improved lipschitz regularization},
  author={Fazlyab, Mahyar and Entesari, Taha and Roy, Aniket and Chellappa, Rama},
  journal={Advances in Neural Information Processing Systems},
  volume={36},
  pages={34451--34464},
  year={2023}
}

@inproceedings{entesari2023reachlipbnb,
  title={ReachLipBnB: A branch-and-bound method for reachability analysis of neural autonomous systems using Lipschitz bounds},
  author={Entesari, Taha and Sharifi, Sina and Fazlyab, Mahyar},
  booktitle={2023 IEEE International Conference on Robotics and Automation (ICRA)},
  pages={1003--1010},
  year={2023},
  organization={IEEE}
}

@article{sharifi2024provable,
  title={Provable Bounds on the Hessian of Neural Networks: Derivative-Preserving Reachability Analysis},
  author={Sharifi, Sina and Fazlyab, Mahyar},
  journal={arXiv preprint arXiv:2406.04476},
  year={2024}
}

@article{entesari2024compositional,
  title={Compositional curvature bounds for deep neural networks},
  author={Entesari, Taha and Sharifi, Sina and Fazlyab, Mahyar},
  journal={arXiv preprint arXiv:2406.05119},
  year={2024}
}

@article{van1996matrix,
  title={Matrix computations (Johns Hopkins studies in mathematical sciences)},
  author={Van Loan, Charles F and Golub, G},
  journal={Matrix Computations},
  volume={5},
  pages={32},
  year={1996}
}

@article{ibrahum2024deep,
  title={Deep learning adversarial attacks and defenses in autonomous vehicles: A systematic literature review from a safety perspective},
  author={Ibrahum, Ahmed Dawod Mohammed and Hussain, Manzoor and Hong, Jang-Eui},
  journal={Artificial Intelligence Review},
  volume={58},
  number={1},
  pages={28},
  year={2024},
  publisher={Springer}
}

@article{javed2024robustness,
  title={Robustness in deep learning models for medical diagnostics: security and adversarial challenges towards robust AI applications},
  author={Javed, Haseeb and El-Sappagh, Shaker and Abuhmed, Tamer},
  journal={Artificial Intelligence Review},
  volume={58},
  number={1},
  pages={12},
  year={2024},
  publisher={Springer}
}

@article{tambon2022certify,
  title={How to certify machine learning based safety-critical systems? A systematic literature review},
  author={Tambon, Florian and Laberge, Gabriel and An, Le and Nikanjam, Amin and Mindom, Paulina Stevia Nouwou and Pequignot, Yann and Khomh, Foutse and Antoniol, Giulio and Merlo, Ettore and Laviolette, Francois},
  journal={Automated Software Engineering},
  volume={29},
  number={2},
  pages={38},
  year={2022},
  publisher={Springer}
}

@article{wang2021beta,
  title={Beta-crown: Efficient bound propagation with per-neuron split constraints for neural network robustness verification},
  author={Wang, Shiqi and Zhang, Huan and Xu, Kaidi and Lin, Xue and Jana, Suman and Hsieh, Cho-Jui and Kolter, J Zico},
  journal={Advances in neural information processing systems},
  volume={34},
  pages={29909--29921},
  year={2021}
}

@article{zhang2018efficient,
  title={Efficient neural network robustness certification with general activation functions},
  author={Zhang, Huan and Weng, Tsui-Wei and Chen, Pin-Yu and Hsieh, Cho-Jui and Daniel, Luca},
  journal={Advances in neural information processing systems},
  volume={31},
  year={2018}
}

@article{xu2020fast,
  title={Fast and complete: Enabling complete neural network verification with rapid and massively parallel incomplete verifiers},
  author={Xu, Kaidi and Zhang, Huan and Wang, Shiqi and Wang, Yihan and Jana, Suman and Lin, Xue and Hsieh, Cho-Jui},
  journal={arXiv preprint arXiv:2011.13824},
  year={2020}
}

@inproceedings{shi2025neural,
  title={Neural network verification with branch-and-bound for general nonlinearities},
  author={Shi, Zhouxing and Jin, Qirui and Kolter, Zico and Jana, Suman and Hsieh, Cho-Jui and Zhang, Huan},
  booktitle={International Conference on Tools and Algorithms for the Construction and Analysis of Systems},
  pages={315--335},
  year={2025},
  organization={Springer}
}

@inproceedings{zhang2019recurjac,
  title={Recurjac: An efficient recursive algorithm for bounding jacobian matrix of neural networks and its applications},
  author={Zhang, Huan and Zhang, Pengchuan and Hsieh, Cho-Jui},
  booktitle={Proceedings of the AAAI Conference on Artificial Intelligence},
  volume={33},
  number={01},
  pages={5757--5764},
  year={2019}
}

@article{szegedy2013intriguing,
  title={Intriguing properties of neural networks},
  author={Szegedy, Christian and Zaremba, Wojciech and Sutskever, Ilya and Bruna, Joan and Erhan, Dumitru and Goodfellow, Ian and Fergus, Rob},
  journal={arXiv preprint arXiv:1312.6199},
  year={2013}
}

@article{huang2021training,
  title={Training certifiably robust neural networks with efficient local lipschitz bounds},
  author={Huang, Yujia and Zhang, Huan and Shi, Yuanyuan and Kolter, J Zico and Anandkumar, Anima},
  journal={Advances in Neural Information Processing Systems},
  volume={34},
  pages={22745--22757},
  year={2021}
}

@article{shi2022efficiently,
  title={Efficiently computing local lipschitz constants of neural networks via bound propagation},
  author={Shi, Zhouxing and Wang, Yihan and Zhang, Huan and Kolter, J Zico and Hsieh, Cho-Jui},
  journal={Advances in Neural Information Processing Systems},
  volume={35},
  pages={2350--2364},
  year={2022}
}

@article{jordan2020exactly,
  title={Exactly computing the local lipschitz constant of relu networks},
  author={Jordan, Matt and Dimakis, Alexandros G},
  journal={Advances in Neural Information Processing Systems},
  volume={33},
  pages={7344--7353},
  year={2020}
}

@article{fazlyab2019efficient,
  title={Efficient and accurate estimation of lipschitz constants for deep neural networks},
  author={Fazlyab, Mahyar and Robey, Alexander and Hassani, Hamed and Morari, Manfred and Pappas, George},
  journal={Advances in neural information processing systems},
  volume={32},
  year={2019}
}

@article{araujo2023unified,
  title={A unified algebraic perspective on lipschitz neural networks},
  author={Araujo, Alexandre and Havens, Aaron and Delattre, Blaise and Allauzen, Alexandre and Hu, Bin},
  journal={arXiv preprint arXiv:2303.03169},
  year={2023}
}

@inproceedings{wang2023direct,
  title={Direct parameterization of lipschitz-bounded deep networks},
  author={Wang, Ruigang and Manchester, Ian},
  booktitle={International Conference on Machine Learning},
  pages={36093--36110},
  year={2023},
  organization={PMLR}
}

@article{pauli2024novel,
  title={Novel quadratic constraints for extending lipsdp beyond slope-restricted activations},
  author={Pauli, Patricia and Havens, Aaron and Araujo, Alexandre and Garg, Siddharth and Khorrami, Farshad and Allg{\"o}wer, Frank and Hu, Bin},
  journal={arXiv preprint arXiv:2401.14033},
  year={2024}
}

@inproceedings{pauli2023lipschitz,
  title={Lipschitz constant estimation for 1d convolutional neural networks},
  author={Pauli, Patricia and Gramlich, Dennis and Allg{\"o}wer, Frank},
  booktitle={Learning for Dynamics and Control Conference},
  pages={1321--1332},
  year={2023},
  organization={PMLR}
}

@inproceedings{hu2020reach,
  title={Reach-sdp: Reachability analysis of closed-loop systems with neural network controllers via semidefinite programming},
  author={Hu, Haimin and Fazlyab, Mahyar and Morari, Manfred and Pappas, George J},
  booktitle={2020 59th IEEE conference on decision and control (CDC)},
  pages={5929--5934},
  year={2020},
  organization={IEEE}
}

@article{huang2019reachnn,
  title={Reachnn: Reachability analysis of neural-network controlled systems},
  author={Huang, Chao and Fan, Jiameng and Li, Wenchao and Chen, Xin and Zhu, Qi},
  journal={ACM Transactions on Embedded Computing Systems (TECS)},
  volume={18},
  number={5s},
  pages={1--22},
  year={2019},
  publisher={ACM New York, NY, USA}
}

@inproceedings{ivanov2019verisig,
  title={Verisig: verifying safety properties of hybrid systems with neural network controllers},
  author={Ivanov, Radoslav and Weimer, James and Alur, Rajeev and Pappas, George J and Lee, Insup},
  booktitle={Proceedings of the 22nd ACM International Conference on Hybrid Systems: Computation and Control},
  pages={169--178},
  year={2019}
}

@inproceedings{ivanov2021verisig,
  title={Verisig 2.0: Verification of neural network controllers using taylor model preconditioning},
  author={Ivanov, Radoslav and Carpenter, Taylor and Weimer, James and Alur, Rajeev and Pappas, George and Lee, Insup},
  booktitle={International Conference on Computer Aided Verification},
  pages={249--262},
  year={2021},
  organization={Springer}
}

@inproceedings{kochdumper2023open,
  title={Open-and closed-loop neural network verification using polynomial zonotopes},
  author={Kochdumper, Niklas and Schilling, Christian and Althoff, Matthias and Bak, Stanley},
  booktitle={NASA Formal Methods Symposium},
  pages={16--36},
  year={2023},
  organization={Springer}
}

\clearpage

\appendix

\section{Proofs}
\subsection{Proof of equation \ref{eq:pointwise2ndOrderBound}}

\label{proofOfeq:pointwise2ndOrderBound}

\begin{proof}
    To prove the desired bound, we utilize the mean-value theorem. Define the scalar function $\phi(t) = f(x + t\delta)$. We have that
    \[
    \phi'(t) = \nabla f(x+t\delta)^\top \delta, \quad \phi''(t) = \delta^\top \nabla^2 f(x+t\delta) \delta.
    \]
    The mean-value theorem guarantees
    \begin{align*}
        \phi(1) &= \phi(0) + \int_0^1 \phi'(t) dt\\
        &= \phi(0) + \phi'(0) + \int_0^1 \left( \phi'(t) - \phi'(0) \right) dt\\
        &= \phi(0) + \phi'(0) + \int_0^1 \left( \int_0^t \phi''(s)ds \right) dt\\
        &= \phi(0) 
        + \phi'(0) 
        + \int_0^1 \left( \int_0^t \left(  \phi''(s) \pm \phi''(0)\right)ds \right) dt\\
        &= \phi(0) 
        + \phi'(0) 
        + \frac{1}{2}\phi''(0)
        + \int_0^1 \left( \int_0^t \left(  \phi''(s) - \phi''(0)\right)ds \right) dt.
    \end{align*}
    Next, we have
    \begin{align*}
        \left|\phi(1) - \phi(0) 
        - \phi'(0) 
        - \frac{1}{2}\phi''(0)\right| 
        &= 
        \left|
        \int_0^1 \left( \int_0^t \left(  \phi''(s) - \phi''(0)\right)ds \right) dt
        \right|\\
        &\leq 
        \int_0^1  \int_0^t \left|  \phi''(s) - \phi''(0)\right|ds dt\\
        &\leq \int_0^1  \int_0^t L_{\phi''} s ds dt\\
        &= \frac{1}{6}L_{\phi''}.
    \end{align*}
    Finally, note that $L_{\phi''} \leq L_{\nabla^2 f} \norm{\delta}_2^3$.
    Putting all these together, we arrive at the desired bound.
\end{proof}

\subsection{Proof of Proposition \ref{prop:splitAndBound}}\label{proofOfProp:splitAndBound}
\begin{proof}
    Expanding the $\sup$ operator over the individual terms, we have
    \begin{align*}
        \bar{f}_2^*(x_c, \varepsilon) &\leq 
        \sup_{\norm{\delta}_2\leq \varepsilon} f(x_c) + 
        \sup_{\norm{\delta}_2\leq \varepsilon} \nabla f(x_c)^\top \delta + 
        \sup_{\norm{\delta}_2\leq \varepsilon} \frac{1}{2}\delta^\top f(x_c)\delta + 
        \sup_{\norm{\delta}_2\leq \varepsilon} \frac{1}{6}L_{\nabla^2 f} \norm{\delta}_2^3\\
        &= f(x_c) + \norm{\nabla f(x_c)}_2 \varepsilon + \frac{1}{2} \relu{\lambda_{\max} \nabla^2 f(x_c)} \varepsilon^2 + \frac{1}{6}L_{\nabla^2 f} \varepsilon^3,
    \end{align*}
    where the $\relu{\lambda_{\max} \nabla^2 f(x_c)}$ shows up because if $\lambda_{\max} \nabla^2 f(x_c) < 0$, then the corresponding $\sup$ would choose $\delta = 0$ to maximize its objective.
\end{proof}

\subsection{Proof of Proposition \ref{prop:thirdIsBetter}} \label{proofOfProp:thirdIsBetter}
\begin{proof}
    We simply solve for when the bound provided by the second-order method is smaller than the one provided by the first-order method. We have:
    \[
    f(x_c) + \norm{\nabla f(x_c)}_2 \varepsilon + \frac{1}{2} \relu{\lambda_{\max} \nabla^2 f(x_c)} \varepsilon^2 + \frac{1}{6}L_{\nabla^2 f} \varepsilon^3 
    \leq
    f(x_c) + \|\nabla f(x_c)\|_2 \varepsilon + \frac{1}{2} L_{\nabla f} \varepsilon^2,
    \] 
    \[
    \frac{1}{2} \relu{\lambda_{\max} \nabla^2 f(x_c)} + \frac{1}{6}L_{\nabla^2 f} \varepsilon 
    \leq
    \frac{1}{2} L_{\nabla f},
    \]
    Which yields the desired result.
\end{proof}

\subsection{Proof of Proposition \ref{prop:secondOrderUpperBoundInfinity}}
\label{sec:proofOfSecondOrderUpperBoundInfinity}

\begin{proof}
    By definition, we have that
    \[
    \bar{f}_2^*(x_c, D) = \max_{\|D^{-1} \delta\|_\infty \leq 1} f(x_c) + \nabla f(x_c)^\top \delta + \frac{1}{2} \delta^\top \nabla^2 f(x_c) \delta + \frac{1}{6} L_{\nabla^2 f} \|\delta\|_2^3
    \]
    We upper bound the right-hand side by taking the maximization individually over each term. We have
    \[
    \max_{\|D^{-1} \delta\|_\infty \leq 1} \nabla f(x_c)^\top \delta 
    = 
      \max_{\| \nu \|_\infty \leq 1} \nabla f(x_c)^\top D\nu
    = 
     \|D \nabla f(x_c)\|_1,
    \]
    where we have used the definition of the dual norm.
    Moreover, we have
    \begin{align*}
        \begin{split}
            \max_{\|D^{-1} \delta\|_\infty \leq 1} \delta^\top \nabla^2 f(x_c) \delta 
            &=
             \max_{\|\nu\|_\infty \leq 1} \nu^\top D\nabla^2 f(x_c) D \nu\\
            &\leq 
             \max_{\|\nu\|_2 \leq \sqrt{n}} \nu^\top D\nabla^2 f(x_c) D \nu\\
            &=
            n \max_{\|\nu\|_2 \leq 1} \nu^\top D\nabla^2 f(x_c) D \nu\\
            &=
            n \paran{\lambda_{\max} \paran{D\nabla^2 f(x_c) D}}_+\\
            &\leq
            n d_{\max}^2   \paran{\lambda_{\max} \nabla^2 f(x_c)}_+,
        \end{split}
    \end{align*}
    where $d_{\max}$ is the largest diagonal element of $D$.
    Finally, we have
    \[
    \max_{\|D^{-1} \delta\|_\infty \leq 1} \norm{\delta}_2^3 = \max_{\|\delta\|_\infty \leq 1} \norm{D\delta}_2^3 =  \norm{D}_F^3,
    \]
    where the last equality holds as $D$ is a diagonal matrix. This concludes the proof.

\end{proof}

\subsection{Proof of Theorem \ref{thm:lipschitzOfCurvatureCompositional}} \label{sec:proofOfLipschitzOfCurvatureCompositional}

\begin{proof}
    Writing the definition of Lipschitz continuity, we have
\begin{align*}
    \| D^2 F_j^I &\circ a^{I - 1} (x) - D^2 F_j^I \circ a^{I - 1} (y) \|_2 \\
    &\leq 
    \underbrace{\bigg\|  D a^{I - 1}(x) ^\top \nabla^2 F_j^I \big( a^{I - 1}(x) \big) D a^{I - 1}(x) -  D a^{I - 1}(y) ^\top \nabla^2 F_j^I \big( a^{I - 1}(y) \big) D a^{I - 1}(y) \bigg\|_2}_{\RNum{1}}\\
    &\quad + \underbrace{\sum_{i=1}^{N_{I - 1}} \bigg\|  \frac{\partial F_j^I(x)}{\partial x_i}\bigg|_{a^{I - 1}(x)} \nabla^2 a^{I - 1}_i(x) - \frac{\partial F_j^I(x)}{\partial x_i}\bigg|_{a^{I - 1}(y)} \nabla^2 a^{I - 1}_i(y) \bigg\|_2}_{\RNum{2}}.
\end{align*}
We can further bound each term. We have
\begin{align*}
    \RNum{1} &= \bigg\|  D a^{I - 1}(x) ^\top \nabla^2 F_j^I \big( a^{I - 1}(x) \big) D a^{I - 1}(x) \\
    &\qquad \pm 
    D a^{I - 1}(x) ^\top \nabla^2 F_j^I \big( a^{I - 1}(y) \big) D a^{I - 1}(y)\\
    &\qquad -  D a^{I - 1}(y) ^\top \nabla^2 F_j^I \big( a^{I - 1}(y) \big) D a^{I - 1}(y) \bigg\|_2\\
    &\leq 
    \bigg\| D a^{I - 1}(x) ^\top \bigg\|_2 
    \cdot \bigg\| \nabla^2 F_j^I \big( a^{I - 1}(x) \big) D a^{I - 1}(x) \\
    &\qquad \qquad \qquad \qquad \quad 
    \pm \nabla^2 F_j^I \big( a^{I - 1}(y) \big) D a^{I - 1}(x)\\
    &\qquad \qquad \qquad \qquad \quad  
    - \nabla^2 F_j^I \big( a^{I - 1}(y) \big) D a^{I - 1}(y) \bigg\|_2\\
    &\qquad 
    + \bigg\|\left( D a^{I - 1}(x) - D a^{I - 1}(y)\right)^\top\bigg\|_2 
    \cdot 
    \bigg\|\nabla^2 F_j^I \big( a^{I - 1}(y) \big)\bigg\|_2
    \cdot
    \bigg\|D a^{I - 1}(y)\bigg\|_2.
\end{align*}
We use the fact that over matrices we have 
$\norm{A}_2 = \norm{A^\top}_2$ to further simplify. We have
\begin{align*}
    \RNum{1}&\leq  
     \bigg\| \nabla^2 F_j^I \big( a^{I - 1}(x) \big) - \nabla^2 F_j^I \big( a^{I - 1}(y) \big) \bigg\|_2
    \cdot \bigg\|D a^{I - 1}(x)\bigg\|_2^2\\
    &\qquad + 
    \bigg\| D a^{I - 1}(x) \bigg\|_2
    \cdot \bigg\| \nabla^2 F_j^I \big( a^{I - 1}(y) \big)\bigg\|_2
    \cdot \bigg\| D a^{I - 1}(x) - D a^{I - 1}(y) \bigg\|_2\\
    &\qquad + \bigg\| D a^{I - 1}(x) - D a^{I - 1}(y)\bigg\|_2 
    \cdot \bigg\|\nabla^2 F_j^I \big( a^{I - 1}(y) \big)\bigg\|_2
    \cdot \bigg\|D a^{I - 1}(y)\bigg\|_2\\
    &\leq \bigg( L_{\nabla^2F_j^I} L_{a^{I - 1}} 
     \big\| D a^{I - 1}(x) \big\|_2^2  \\
    &\qquad+ L_{D a^{I - 1}} \big\| D a^{I - 1}(x)  \big\|_2 
    \cdot \big\| \nabla^2 F_j^I \big( a^{I - 1}(y) \big) \big\|_2\\
    &\qquad + L_{D a^{I - 1}} \big\| D a^{I - 1}(y) \big\|_2 
    \cdot \big\| \nabla^2 F_j^I \big( a^{I - 1}(y) \big) \big\|_2
    \bigg) \big\| x - y \big\|_2\\
    &\leq \left( L_{\nabla^2 F_j^I} L_{a^{I - 1}}^3 
    +  2L_{D a^{I - 1}} L_{a^{I - 1}} L_{\nabla F_j^I} 
    \right) \big\| x - y \big\|_2.
\end{align*}

Furthermore
\begin{align*}
    \RNum{2} &= \sum_{i=1}^{N_{I - 1}} \bigg\|  \frac{\partial F_j^I(x)}{\partial x_i}\bigg|_{a^{I - 1}(x)} \nabla^2 a^{I - 1}_i(x) \\
    &\qquad \quad 
    \pm \frac{\partial F_j^I(x)}{\partial x_i}\bigg|_{a^{I - 1}(y)} \nabla^2 a^{I - 1}_i(x)
    - \frac{\partial F_j^I(x)}{\partial x_i}\bigg|_{a^{I - 1}(y)} \nabla^2 a^{I - 1}_i(y) \bigg\|_2\\
    &\leq 
    \sum_{i=1}^{N_{I - 1}} \Bigg( \Bigg|  \frac{\partial F_j^I(x)}{\partial x_i}\bigg|_{a^{I - 1}(x)} - \frac{\partial F_j^I(x)}{\partial x_i}\bigg|_{a^{I - 1}(y)}\Bigg| 
    \cdot \bigg\|\nabla^2 a^{I - 1}_i(x) \bigg\|_2\\
    &\qquad \qquad 
    + \bigg| \frac{\partial F_j^I(x)}{\partial x_i}\bigg|_{a^{I - 1}(y)} \bigg|
    \cdot \bigg\| \nabla^2 a^{I - 1}_i(x) - \nabla^2 a^{I - 1}_i(y)\bigg\|_2 \Bigg)\\
    &\leq \sum_{i=1}^{N_{I - 1}} \left( L_{\frac{\partial F_j^I(x)}{\partial x_i}} L_{a^{I - 1}} L_{\nabla a^{I - 1}_i} + L_{\nabla^2 a^{I - 1}_i} \sup_x |\frac{\partial F_j^I(x)}{\partial x_i}\bigg|_x |  \right) \|x - y\|_2.
\end{align*}
Putting cases $\RNum{1}$ and $\RNum{2}$ together we get
\begin{align*}
    \| D^2 F_j^I & \circ a^{I - 1} (x) - D^2 F_j^I \circ a^{I - 1} (y) \|_2 \leq \\
    & \quad \Bigg( L_{\nabla^2 F_j^I} L_{a^{I - 1}}^3
    + 2L_{D a^{I - 1}} L_{a^{I - 1}} L_{\nabla F_j^I}\\
    &\qquad 
    + \sum_{i=1}^{N_{I - 1}} \left( L_{\frac{\partial F_j^I(x)}{\partial x_i}} L_{a^{I - 1}} L_{\nabla a^{I - 1}_i} + L_{\nabla^2 a^{I - 1}_i} \sup_x |\frac{\partial F_j^I(x)}{\partial x_i}\bigg|_x | \right)
    \Bigg) \big\| x - y \big\|_2
\end{align*}
\end{proof}

\subsection{Derivation of equations \ref{eq:supOfDerivative}, \ref{eq:lipschitzOfPartialDerivative}, \ref{eq:lipschitzOfGradient}, and
\ref{eq:lipschitzOfCurvature} }
\label{sec:appendixIndividualLipschitzTerms}
As stated, we have
\[
DF^I(x) = \text{Diag}(\sigma'(W^Ix + b^I))W^I.
\]
Subsequently, we can derive each of the desired terms. First, we have that 
\[
\frac{\partial F^I_j(x)}{\partial x_i} = W^I_{ji} \sigma'(W^I_{j,:} x + b^I_j).
\]
Consequently, as $\sup_x \sigma'(x) \leq L_\sigma$, we have that 
\[
\sup \frac{\partial F^I_j(x)}{\partial x_i} = L_\sigma |W^I_{ji}|.
\]
Moreover, using the Lipschitz property of $\sigma'$ we have that 
\[
L_{\frac{\partial F^I_j(x)}{\partial x_i}} = L_{\sigma'} 
\norm{W^I_{j,:}}_2 
|W^I_{ji}|.
\]
To derive $L_{\nabla F^I_j}$, we simply write the definition of Lipschitz continuity for $\nabla F^I_j$. We have
\begin{align*}
    \norm{\nabla F^I_j(x) - \nabla F^I_j(y)}_2 
    &=
    \norm{\paran{\sigma'(W^I_{j,:}x + b^I_j) - \sigma'(W^I_{j,:}y + b^I_j)}
    {W^I_{j,:}}^\top
    }_2\\
    &\leq 
    \norm{\sigma'(W^I_{j,:}x + b^I_j) - \sigma'(W^I_{j,:}y + b^I_j)
    }_2
    \cdot
    \norm{W^I_{j,:}}_2\\
    &\leq
    L_{\sigma'} \norm{W^I_{j,:}}_2^2.
\end{align*}
This result is also given if we were to use Theorem 3.9 of \cite{entesari2024compositional}.
Going one step further, we have
\[
\nabla^2 F^I_j = \sigma''(W^I_{j,:}x + b^I_j) {W^I_{j,:}}^\top W^I_{j,:}.
\]
As $\norm{{W^I_{j,:}}^\top W^I_{j,:}}_2 = \norm{W^I_{j,:}}_2^2$, this yields the desired result for $L_{\nabla^2 F^I_j}$

\subsection{Derivation of Lipschitz terms for $\hat{F}$}\label{sec:appendixIndividualLipschitzTermsForSecondFunction}
$\hat{F} = c^\top W^L \sigma (W^{L - 1} x + b^{L - 1}) + c^\top b^L $
The basis for the proofs of this section are \cite{entesari2024compositional} and we only establish the connection.
We have
\[
\nabla \hat{F}(x) = {W^{L - 1}}^\top \text{Diag}(\sigma'(W^{L - 1}x+b^{L - 1})){W^L}^\top c
\]
We can rewrite this as
\[
\nabla \hat{F}(x) = {W^{L - 1}}^\top \text{Diag}({W^L}^\top c)\sigma'(W^{L - 1}x+b^{L - 1}).
\]
This simply follows from the exchange of terms between a diagonal matrix and multiplication by a vector. As instructed in the text, we define $A = {W^{L - 1}}^\top \text{Diag}({W^L}^\top c)$. Now we have $\nabla \hat{F}(x) = A \sigma'(W^{L - 1}x + b^{L - 1})$.
From this, we first obtain that 
\[
\sup \frac{\partial \hat{F}(x)}{\partial x_i} = L_\sigma \sum_j |A_{ij}|,
\]
and then
\[
L_{\frac{\partial \hat{F}(x)}{\partial x_i}} = L_{\sigma'} 
\sum_j |A_{ij}| \norm{W^{L - 1}_{j,:}}_2.
\]
The naive Lipschitz bound establishes that 
\[
 L_{\nabla \hat{F}} = L_{\sigma'} \norm{A}_2 \norm{W^{L - 1}}_2.
\]
Finally, the rewritten $\nabla \hat{F}(x) = A \sigma'(W^{L - 1}x + b^{L - 1})$ is in the form of a standard neural network layer as described in \cite{entesari2024compositional}. As such we can utilize Theorem 3.8 therein to acquire $L_{\nabla^2 \hat{F}}$, by simply noting that $L_{D \nabla \hat{F}} = L_{\nabla^2 \hat{F}}$.

\end{document}